\documentclass[runningheads]{llncs}
\usepackage{makeidx}  
\usepackage{epstopdf}
\usepackage{graphicx}
\usepackage{subfigure}
\usepackage{lineno}

\usepackage[T1]{fontenc}
\usepackage[english]{babel}
\usepackage{dsfont}
\usepackage{amssymb}
\usepackage{amsmath}
\usepackage{amsfonts}

\usepackage{pgfplots}
 
\pgfplotsset{compat = newest}

\usepackage{bbm}

\newcommand{\relu}{\texttt{ReLU}}
\newcommand{\maxpool}{\texttt{MaxPool}}
\newcommand{\minpool}{\texttt{MinPool}}
\newcommand{\lrelu}{\texttt{LeakyReLU}}
\newcommand{\prelu}{\texttt{PReLU}}

\newcommand{\realset}{\mathds{R}}
\newcommand{\lattice}{{\cal L}}
\newcommand{\fun}{\mathcal{F}}

\begin{document}


\title{MorphoActivation: Generalizing ReLU activation function by mathematical morphology}
\titlerunning{MorphoActivation}

\author{Santiago Velasco-Forero \and Jes\'us Angulo}
\authorrunning{Santiago Velasco-Forero \and Jes\'us Angulo }   
%
\tocauthor{Santiago Velasco-Forero , Jes\'us Angulo}
%
\institute{MINES  Paris,  PSL University,\\
Centre for mathematical morphology (CMM), France\\
\email{santiago.velasco@mines-paristech.fr, jesus.angulo@mines-paristech.fr}}

\maketitle              

\begin{abstract}
This paper analyses both nonlinear activation functions and spatial max-pooling for Deep Convolutional Neural Networks (DCNNs) by means of the algebraic basis of mathematical morphology. Additionally, a general family of activation functions is proposed by considering both max-pooling and nonlinear operators in the context of  morphological representations. Experimental section validates the goodness of our approach on classical benchmarks for supervised learning by DCNN.
\keywords{Matheron's representation theory; activation function; mathematical morphology; deep learning.}
\end{abstract}

\section{Introduction}
\label{sec:Intro}
Artificial neural networks were introduced as mathematical models for biological neural networks \cite{mcculloch1943logical}. The basic component is a \emph{linear perceptron} which is a linear combination of weights with biases followed by a nonlinear function called \emph{activation function}. Such components (usually called a \emph{layer}) can then be concatenated eventually leading to very complex functions named \emph{deep artificial neural networks (DNNs)} \cite{goodfellow2016deep}. Activation function can also be seen as an attached function between two layers in a neural network. Meanwhile, in order to get the learning in a DNNs, one needs to update the weights and biases of the neurons on the basis of the error at the output. This process involves two steps, a \emph{Back-Propagation} from prediction error and a Gradient Descent Optimization to update parameters \cite{goodfellow2016deep}. The most famous activation function is the Rectified Linear Unit (ReLU) proposed by \cite{nair2010rectified}, which is simply defined as  $\relu(x)=\max(x,0)$. A clear benefit of ReLU is that both the function itself and its derivatives are easy to implement and computationally inexpensive. However,  ReLU has a potential loss during optimization because the gradient is zero when the unit is not active. This could lead to cases where there is a gradient-based optimization algorithm that will not adjust the weights of a unit that was never initially activated. An approach purely computational motivated to alleviate potential problems caused by the hard zero activation of ReLU, proposed a leaky ReLU activation \cite{maas2013rectifier}: $\lrelu(x)=\max(x,.01x)$. A simple generalisation is the 
Parametric ReLU proposed by \cite{he2015delving}, defined as $\prelu_{\beta}(x)=\max(x,\beta x)$, where $\beta \in \realset$ is a learnable parameter. In general, the use of \emph{piecewise-linear functions} as activation function has been initially motivated by neurobiological observations; for instance, the inhibiting effect of the activity of a visual-receptor unit on the activity of the neighbouring units can be modelled by a line with two segments \cite{hartline1957inhibitory}. 
On the other hand, for the particular case of structured data as images, a translation invariant DNN called Deep Convolutional Neural Networks (DCNN) is the most used architecture. In the conventional DCNN framework interspersed convolutional layers and pooling layers to summarise information in a hierarchical structure. The common choice is the pooling by a maximum operator called \emph{max-pooling}, which is particularly well suited to the separation of features that are very sparse \cite{boureau2010theoretical}. 

As far as these authors know, that morphological operators have been used in the context of DCNNs following the paradigm of replacing lineal convolutions by non-linear morphological filters \cite{ritter1996introduction,LEMD2022,islam2020extending,mondal2020image,Hermary2022}, or hybrid variants between linear and morphological layers \cite{pessoa2000neural,sussner2020extreme,hernandez2020hybrid,valle2020reduced}. Our contribution is more in the sense of \cite{franchi2020deep} where the authors show favourable results in quantitative performance for some applications when seeing the max-pooling operator as a dilation layer. However, we go further to study both nonlinear activation and max-pooling operators in the context of morphological representation theory of nonlinear operators. Finally, in the experimental section, we compare different propositions in a practical case of training a multilayer CNNs for classification of images in several databases.


\section{ReLU activation and max-pooling are morphological dilations}
\label{sec:Motivation}

\subsection{Dilation and Erosion}

Let us consider a complete lattice $(\lattice,\leq)$, where and $\bigvee$ and $\bigwedge$ are respectively its supremum and infimum. A lattice operator $\psi: \lattice \to \lattice  $ is called \emph{increasing} operator (or isotone) is if it is order-preserving, i.e., $\forall X, Y$, $X \leq Y \implies \psi(X) \leq \psi(Y)$. Dilation $\delta$ and erosion $\varepsilon$ are lattice operators which are increasing and satisfy
\begin{eqnarray}\nonumber
\delta\left( \bigvee_{i\in J} X_i\right) = \bigvee_{i\in J} \delta\left(  X_i\right); & \:\:\:\:  & \varepsilon\left( \bigwedge_{i\in J} X_i\right) =  \bigwedge_{i\in J} \varepsilon\left(  X_i\right).
\end{eqnarray}
Dilation and erosion can be then composed to obtain other operators \cite{heijmans1991theoretical}.
In this paper, we also use morphological operators on the lattice of functions $\fun(\realset^n,\bar{\realset})$ with the standard partial order $\leq$. The sup-convolution and inf-convolution of function $f$ by structuring function $g$ are given by
\begin{eqnarray}\label{DilationFunctions}
(f \oplus g)(x) = \delta_{g}(f)(x) & := & \sup_{y\in \realset^{n}}\left\{f(x-y) + g(y)\right\}, \\ \label{ErosionFunctions}
(f \ominus g)(x) = \varepsilon_{g}(f)(x) & := & \inf_{y\in \realset^{n}}\left\{f(x+y) - g(y)\right\}.
\end{eqnarray}

\subsection{ReLU and max-pooling}

Let us now consider the standard framework of one-dimensional\footnote{The extension to $d$-dimensional functions is straightforward.} signals on DCNNs where any operator is applied on signals $f(x)$ supported on a discrete grid subset of $\mathbb{Z}$. The \emph{ReLU} activation function \cite{nair2010rectified} applied on every pixel $x$ of an image $f$ is defined as 
\begin{equation}
\relu(f(x))  := \max\left(0,f(x)\right). 
\end{equation}


The \emph{Max-pooling operator} of \emph{pooling size}  $R$ and strides $K$, maps an image $f$ of $n$ pixels onto an image of $n':=\lfloor \frac{n-R}{K} + 1 \rfloor$ by taking the local maxima in a neighbour of size $R$, and moving the window $K$ elements at a time, skipping the intermediate locations:
\begin{equation}
\maxpool_{R}(f)(x) = \delta^{\maxpool}_{R}(f)(x) := \bigvee_{y\in W_{R}(x)}\{ f(K \cdot x - y) \} .
\end{equation}
where $W_R(y)=0$ if $y$ belongs to the neighbour of size $R$ centred in $x$ and $-\infty$ otherwise.
There are other operations in DCNN which use the maximum operation as main ingredient, namely the \emph{Maxout layer}~\cite{goodfellow2013maxout} and the \emph{Max-plus layer (morphological perceptron)}~\cite{charisopoulos2017morphological,zhang2019max}.

From the definition of operators, it is straightforward to prove the following proposition
\begin{proposition}
ReLU activation function and max-pooling are dilation operators on the lattice of functions.
\end{proposition}
\begin{proof}
Using the standard partial ordering $\leq$, we note that both ReLU  and max-pooling are increasing:
		\begin{equation*}
		f \leq g \implies \relu(f) \leq \relu(g); \:\:\: \delta^{\maxpool}_{R}(f) \leq \delta^{\maxpool}_{R}(g).
		\end{equation*}
They commute with supremum operation
		\begin{equation*}
		{\relu}(f \vee g) = {\relu}(f ) \vee {\relu}(g) ; \:\:\: \delta^{\maxpool}_{R}(f \vee g) = \delta^{\maxpool}_{R}(f ) \vee \delta^{\maxpool}_{R}(g).
		\end{equation*}
\end{proof}
These two operators are both also \emph{extensive}, i.e., $f \leq \delta(f)$. $\relu$ is also idempotent, i.e., $\relu (\relu(f))=\relu(f)$. Then $\relu$ is both a dilation and a closing.

\textbf{Remark 1: Factoring activation function and pooling.} The composition of dilations in the same complete lattice can often be factorized into a single operation. One can for instance define a nonlinear activation function and pooling dilation as
\begin{equation*}
\delta^{\text{ActPool}}_{R;\alpha}(f)(x) :=  \bigvee_{y\in W_{R}(x)}\left\{ \max\left(0, f(K\cdot x - y) + \alpha \right) \right\},
\end{equation*}
where $W$ denotes a local neighbour, usually a square of side $R$. Note that that analysis does not bring any new operator, just the interpretation of composed nonlinearities as a dilation.

\bigskip 
\textbf{Remark 2: Positive and negative activation function, symmetric pooling.} More general ReLU-like activation functions also keep a negative part. Let us consider the two parameters $\beta^{+},\beta^{-} \in \realset$, we define \emph{$(\beta^{+},\beta^{-})$-ReLU} as
\begin{equation*}
{\relu}_{\beta^{+},\beta^{-}}(f)(x) := 
\left\{
\begin{array}{ll}
\beta^{+}f(x) & \text{if}\:\:\: f(x) >0 \\
\beta^{-}f(x) & \text{if}\:\:\: f(x) \leq 0
\end{array}
\right.
\end{equation*}
In the case when $\beta^{+}\geq \beta^{-}$, one has
\begin{equation}\label{eq:PosNegRel}
{\relu}_{\beta^{+},\beta^{-}}(f)(x)  = \max\left(\beta^{-}f(x),\beta^{+}f(x)\right).
\end{equation}
Note that the Leaky ReLU~\cite{maas2013rectifier} corresponds to $\beta^{+}=1$ and $\beta^{-}=0.01$. The Parametric ReLU~\cite{he2015delving} takes $\beta^{+}=1$ and $\beta^{-}=\theta$ learned along with the other neural-network parameters. 
More recently~\cite{Ma2021} both  $\beta^{+}$ and $\beta^{-}$  
are learned in the \emph{ACtivateOrNot} (ACON) activation function, where a softmax is used to approximate the maximum operator.

Usually in CNNs, the \emph{max-pooling} operator is used after activation, i.e., 
$\delta^{\text{MaxPool}}_{R}( {\relu}_{\beta^{+},\beta^{-}}(f))$, 
which is spatially enlarging the positive activation and removing the negative activation. It does not seem coherent with the goal of using the pooling to increase spatial equivariance and hierarchical representation of information. It is easy to \emph{"fix"} that issue by using a more symmetric pooling based on taking the positive and negative parts of a function. Given a function $f$, it can be expressed in terms of its positive $f^{+}$ and negative parts $f^{-}$, i.e.,  $f = f^{+} - f^{-}$, with $f^{+}(x) = \max(0,f(x))$ and $f^{-}(x) = \max(0,-f(x))$,
where both $f^{+}$ and $f^{-}$ are non-negative functions. We can now define a \emph{positive and negative max-pooling}. The principle is just to take a max-pooling to each part and recompose, i.e., 
\begin{eqnarray} \label{eq:SelfDualRelu}
\delta^{\maxpool^{+}_{-}}_{R}(f)(x)  &  =   &  \delta^{\maxpool}_{R}(f^{+})(x) - \delta^{\maxpool}_{R}(f^{-})(x)  \\ 
& = & \delta^{\maxpool}_{R}(\max(0,f))(x) + \varepsilon^{\minpool}_{R}(\min(0,f))(x). \nonumber
\end{eqnarray}
We note that \eqref{eq:SelfDualRelu} is \emph{self-dual} and related to the dilation on an inf-semilattice~\cite{Keshet00}. However, in the general case of \eqref{eq:SelfDualRelu} by learning both $\beta^{-},\beta^{+}$,
\begin{equation}\label{eq:SelfDualPosNegPar}
\delta^{\maxpool^{+}_{-}}_{\beta^{+},\beta^{-},R}(f)(x) =\delta^{\maxpool}_{R}(\max(0,\beta^{-}f))(x) + \varepsilon^{\minpool}_{R}(\min(0,\beta^{+}f))(x)
\end{equation}
is not always self-dual.




\section{Algebraic theory of minimal representation for nonlinear operators and functions}
\label{sec:Theory}

In the following section, we present the main results about representation theory of nonlinear operators from Matheron \cite{matheron1974random}, Maragos \cite{Maragos89} and Bannon-Barrera \cite{banon1991minimal} (MMBB).

\subsection{MMBB representation theory on nonlinear operators}

Let us consider a translation-invariant (TI) increasing operator $\Psi$. The domain of the functions considered here is either $E=\realset^n$ or $E=\mathbb{Z}^n$, with the additional condition that we consider only closed subsets of $E$. We consider first the set operator case applied on ${\cal P}(E)$ and functions $f: E \to \realset\cap \infty$. 

\bigskip
\noindent
\textbf{Kernel and basis representation of TI increasing set operators.} The kernel of the TI operator $\Psi$ is defined as the following collection of input sets~\cite{matheron1974random}:
$\text{Ker}(\Psi) =$ $\left\{A \subseteq E \: : \: \mathbf{0}\in \Psi(A) \right\}$,
where $\mathbf{0}$ denotes the origin of $E$. 

\begin{theorem}[Matheron (1975)~\cite{matheron1974random}]
	Consider set operators on ${\cal P}(E)$. Let $\Psi : {\cal P}(E) \to {\cal P}(E) $ be a TI increasing set operator. Then
	\begin{equation*}
	\Psi(X) = \bigcup_{A\in \text{Ker}(\Psi)} X\ominus A  = \bigcap_{B\in \text{Ker}(\bar{\Psi})} X\oplus \check{B} .
	\end{equation*}	
	where the dual set operator is $\bar{\Psi}(X) = $ $[\Psi(X^c)]^c$ and $\check{B}$ is the transpose structuring element.  
\end{theorem}

The kernel of $\Psi$ is a partially ordered set under set inclusion which has an infinity number of elements. In practice, by the property of absorption of erosion, that means that the erosion by $B$ contains the erosions by any other kernel set larger than $B$ and it is the only one required when taking the supremum of erosions. The morphological basis of $\Psi$ is defined as the minimal kernel sets~\cite{Maragos89}:
\begin{equation*}
\text{Bas}(\Psi) = \left\{M \in \text{Ker}(\Psi) \: : \: [ A  \in \text{Ker}(\Psi)\:\: \text{and} \:\: A\subseteq M] \implies A=M \right\}.
\end{equation*}

A sufficient condition for the existence of $\text{Bas}(\Psi)$ is  for $\Psi$ to be an upper semi-continuous operator. We also consider closed sets on ${\cal P}(E)$.

\begin{theorem}[Maragos (1989)~\cite{Maragos89}]
	Let $\Psi : {\cal P}(E) \to {\cal P}(E)$ be a TI, increasing and upper semi-continuous set operator\footnote{Upper semi-continuity meant with respect to the hit-miss topology. Let $(X_n)$ be any decreasing sequence of sets that converges monotonically to a limit set $X$,\emph{i.e.},$X_{n+1}\subseteq X_{n} \forall n$ and $X =\cap_{n}X_n$; that is denoted  by $X_{n} \downarrow X$.\\
	An increasing set operator $\Phi$ on $\fun(E)$ is upper semi-continuous if and only if $X_{n} \downarrow X$ implies that $\Phi(X_n) \downarrow \Phi(X)$.
	} . Then
	\begin{equation*}
	\Psi(X) = \bigcup_{M\in \text{Bas}(\Psi)} X\ominus M  = \bigcap_{N\in \text{Bas}(\bar{\Psi})} X\oplus \check{N} .
	\end{equation*}	
\end{theorem}


\bigskip
\noindent
\textbf{Kernel and basis representation of TI increasing operators on functions.} Previous set theory was extended~\cite{Maragos89} to the case of mappings on functions $\Psi(f)$ and therefore useful for signal or grey-scale image operators. We focus on the case of closed functions $f$, i.e., its epigraph is a closed set. In that case, the dual operator is $\bar{\Psi}(f) =$ $-\Psi(-f)$ and the transpose function is $\check{f}(x)=$ $f(-x)$. 
Let 
\begin{equation*}
\text{Ker}(\Psi) :=\left\{f \: : \Psi(f)(\mathbf{0}) \geq 0 \right\}
\end{equation*}
be the kernel of operator $\Psi$. As for the TI set operators, a basis can be obtained from the kernel functions as its minimal elements with respect to the partial order $\leq$, i.e.,
\begin{equation*}
\text{Bas}(\Psi) := \left\{g \in \text{Ker}(\Psi) \: : \: [ f  \in \text{Ker}(\Psi)\:\: \text{and} \:\: f\leq g] \implies f=g \right\}.
\end{equation*}
This collection of functions can uniquely represent the operator. 

\begin{theorem}[Maragos (1989)~\cite{Maragos89}]
	Consider an upper semi-continuous operator $\Psi$ acting on an upper semi-continuous function \footnote{A function $f:\realset^{n} \to \bar{\realset}$ is \emph{upper semi-continuous} (u.s.c) (resp. lower semi-continuous (l.s.c.)) if and only if, for each $x\in \realset^{m}$ and $t \in \bar{\realset}, f(x)<t$ (resp. $f(x)>t$) implies that $f(y)<t$ (resp. $f(y)<t$) for all in some neighbourhood of $x$. Similarly, $f$ is u.s.c. (resp. l.s.c.) if and only if all its level sets are closed (resp. open) subsets of $\realset^{n}$. A function is continuous iff is both u.s.c and l.s.c.}  $f$. 
	Let $\text{Bas}(\Psi) =$ $\{ g_i\}_{i\in I}$ be its basis and $\text{Bas}(\bar{\Psi})=$ $\{ h_j\}_{j\in J}$ the basis of the dual operator. If $\Psi$ is a TI and increasing operator then it can be represented as 
	\begin{eqnarray}
	\Psi(f)(x) & = & \sup_{i\in I} \: (f\ominus g_i)(x) =
	\sup_{i\in I} \:  \inf_{y\in \mathbb{R}^n}\left\{f(x+y) - g_i(y)\right\} \\ 
	 & = & \inf_{j\in J} \: (f\oplus \check{h}_j)(x) =	
	 \inf_{j\in J} \:  \sup_{y\in \mathbb{R}^n}\left\{f(x-y) + \check{h}_j(y)\right\} 
	\end{eqnarray}
	The converse is true. Given a collection of functions ${\cal B} =$ $\{g_i\}_{i\in  I}$ such that all elements of it are minimal in $({\cal B},\leq)$, the operator $\Psi(f) =$ $\sup_{i\in I}$ $\{f\ominus g_i\}$ is a TI increasing operator whose basis is equal to ${\cal B}$. 
\end{theorem}

For some operators, the basis can be very large (potentially infinity) and even if the above theorem represents exactly the operator by using a full expansion of all erosions, we can obtain an approximation based on smaller collections or truncated bases ${\cal B} \subset$ $\text{Bas}(\Psi)$ and $\bar{{\cal B}} \subset$ $\text{Bas}(\bar{\Psi})$. Then, from the operators
$\Psi_{l}(f) = \sup_{g\in {\cal B}} \{ f\ominus g \}$ and $\Psi_{u}(f) = \inf_{h\in\bar{{\cal B}}} \{ f\oplus \check{h} \}$
the original $\Psi$ is bounded from below and above, i.e.,
$\Psi_{l}(f) \leq $ $\Psi(f) \leq $ $\Psi_{u}(f)$. Note also that in the case of a non minimal representation by a subset of the kernel functions larger than the basis, one just gets a redundant still satisfactory representation.

\bigskip

The extension to TI \emph{non necessarily increasing mappings} was presented by Bannon and Barrera in \cite{banon1991minimal}, which involves a supremum of an operator involving an erosion and an anti-dilation. This part of the Matheron-Maragos-Bannon-Barrera (MMBB) theory is out of the scope of this paper.


\subsection{Max-Min representation for Piecewise-linear functions}

Let us also remind the fundamental results from the representation theory by Ovchinnikov~\cite{Ovchinnikov01,Ovchinnikov02} which is rooted in a Boolean and lattice framework and therefore related to the MMBB theorems. Just note that here we focus on a representation for functions and previously it was a representation of operators on functions. Let $f$ be a smooth function on a closed domain $\Omega \subset \realset^n$. We are going to represent it by a family of affine linear functions $g_t$ which are tangent hyperplanes to the graph of $f$. Namely, for a point $t\in \Omega$, one defines  
\begin{equation}\label{AffineLinearFunctions}
g_t(x) = \left\langle \nabla f(t), x-t\right\rangle + f(t), \:\:\: x\in \Omega,
\end{equation}
where $\nabla f(t)$ is the gradient vector of $f$ at $t$. We have the following general result about the representation of piecewise-linear (PL) functions as max-min polynomial of its linear components. 

\begin{theorem}[\cite{gorokhovik1994piecewise}\cite{bartels1995continuous}\cite{Ovchinnikov02}]
	Let $f$ be a PL function on a closed convex domain $\Omega\subset \realset^n$ and $\{g_1=\beta_1 x +\alpha_1, \cdots , g_d = \beta_d x +\alpha_d\}$ be the set of the $d$ linear components of $f$, with $\beta_i,\alpha_i \in \realset^n$. There is a family $\{K_i\}_{i\in I}$ of subsets of set $\{1, \cdots, d\}$ such that
	\begin{equation}
	f(x) = \bigvee_{i\in I} \bigwedge_{j\in K_i} g_{j}(x), \:\: x\in \Omega.
	\end{equation}
	Conversely, for any family of distinct linear functions $\{g_1, \cdots , g_d\}$ the above formula defines a PL function.
\end{theorem}
The expression is called a max-min (or lattice) polynomial in the variable $g_i$. We note that a PL function $f$ on $\Omega$ is a ``selector'' of its components $g_i$, i.e., $\forall x\in \Omega$ there is an $i$ such that $f(x)=g_i(x)$. The converse is also true, with functions $\{ g_i\}$ linearly ordered over $\Omega$~\cite{Ovchinnikov02}.

Let us also mention that from this representation we can show that a PL function is representable as a difference of two concave (equivalently, convex)-PL functions~\cite{Ovchinnikov02}. More precisely, let note $h_i(x) = \inf_{j\in K_i} g_j(x)$, with $h_i$ a concave function. We are reminded that sums and minimums of concave functions are concave. One have $h_i = $ $\sum_k h_k - \sum_{k \neq i} h_k$, therefore
\begin{equation*}
f = \bigvee_{i\in I} h_i = \sum_k h_k - \bigwedge_{i\in I} \sum_{k \neq i} h_k.
\end{equation*}

\section{Morphological universal activation functions}
\label{sec:activation}

Using the previous results, we can state the two following results for the activation function and the pooling by increasing operators. Additionally, a proposed layer used in the experimental section is formulated.

\subsection{Universal representation for activation function and pooling}

\begin{proposition}
Any piecewise-linear activation function $\sigma: \realset \to \realset$ can be universally expressed as
\begin{eqnarray}\label{GeneralActivation}
\sigma(x) =\bigwedge_{j\in J}\left[ \bigvee_{i\in I}\left\{\beta_i^j x + \alpha_i^j\right\}\right] 
= \bigwedge_{j\in J} p_{j}(x)
\end{eqnarray}
where $p_{j} = \bigvee_{i\in I}\left\{\beta_i^j x + \alpha_i^j\right\}$ is a PL convex function.
\end{proposition}

\begin{proposition}[Pooling]
Any increasing pooling operator $\pi: \realset^n \to \realset^{n'}$ can be universally expressed as
\begin{equation}
\pi(f)(x) = \bigwedge_{j\in J} \left[\delta_{b_j}(f) \right (K\cdot x)],
\end{equation}
where $\{b_j\}_{j\in J}$ is a family of structuring functions defining by transpose the basis of the dual operator to $\pi$.
\end{proposition}

In both cases, there is of course a dual representation using the maximum of erosions. The dilation operator of type $z \mapsto \bigvee_i [\beta_i z + \alpha_i]$ plays a fundamental role in multiplicative morphology~\cite{heijmans1991theoretical}. 

\bigskip
\textbf{Remark: Tropical polynomial interpretation.}
The max-affine function $p_{j} = \bigvee_{i\in I}\left\{\beta_i^j z + \alpha_i^j\right\}$ is a tropical \footnote{\emph{Tropical geometry} is the study of polynomials and their geometric properties when addition is replaced with a minimum operator and multiplication is replaced with ordinary addition.} polynomial such that in that geometry, the degree of the polynomial corresponds to the number of pieces of the PL convex function. 
The set of such polynomials constitutes the semiring $\realset_{\max}$ of tropical
polynomials. Tropical geometry in the context of lattice theory and neural networks is an active area of research 
\cite{maragos2020multivariate} 
\cite{maragos2021tropical}
\cite{dimitriadis2022advances}, however those previous works have not considered the use of minimal representation of tropical polynomials as generalised activation functions.

\


\textbf{Remark: Relationships to other universal approximation theorems.} These results on universal representation of layers in DCNN are  related to study the capacity of neural networks to be universal approximators for smooth functions. 
For instance, both maxout networks~\cite{goodfellow2013maxout} and max-plus networks~\cite{zhang2019max} can approximate arbitrarily well any continuous function on a compact domain. The proofs are based on the fact that~\cite{wang2004general} continuous PL functions can be expressed as a difference of two convex PL functions, and each convex PL can be seen as the maximum of affine terms. 

Tropical formulation of ReLU networks has shown that a deeper network is exponentially more expressive than a shallow network~\cite{Zhang2018tropical}. To explore the expressiveness of networks with our universal activation function and pooling layer respect to the deepness of DCNN is therefore a fundamental relevant topic for future research.








\subsection{MorphoActivation Layer}

We have now all the elements to justify why in terms of universal representation theory of nonlinear operators ReLU and max-pooling can be replaced by a more general nonlinear operator defined by a morphological combination of activation function, dilations and downsampling, using a max-plus layer or its dual. 

More precisely, we introduce two alternative architectures of the MorphoActivation layer (Activation and Pooling Morphological Operator) $f \mapsto $ $\Psi^{\text{Morpho}}: \realset^n \to \realset^{n'}$ either by composition $[\pi \circ \sigma(f)](x)$ or $[\sigma \circ \pi(f)](x)$ as follows:
\begin{equation}\label{MorphoAct1}
	\Psi^{\text{Morpho}}_{1}(f) = 
	\bigwedge_{1\leq j \leq M}\left\{ \delta^{\maxpool}_{R,b_j}\left(\bigvee_{1\leq i \leq N}( \beta_{i}^{j} f + \alpha_{i}^{j})\right) \right\},
\end{equation}
\begin{equation}\label{MorphoAct2}
	\Psi^{\text{Morpho}}_{2}(f) = 
	\bigwedge_{1\leq i \leq N}\left\{ 
	\bigvee_{1\leq j \leq M} \left( \beta_{i}^{j} \delta^{\maxpool}_{R,b_i}(f) +  \alpha_{i}^{j} \right)
	\right\},
\end{equation}
where 
\begin{equation*}
	\left\{
	\begin{array}{l}
	\delta^{\maxpool}_{R,b_j}(f)(x) = \delta_{b_j}(f)(R \cdot x), \:\:\:\: \text{with} \\ \\
	\delta_{b_j}(f)(x)  = (f\oplus b_j)(x) =  \bigvee_{y\in W}\{ f(x - y) + b_j(y) \} 
	\end{array}
	\right.	
\end{equation*}

In the context of an end-to-end learning DCNN, the parameters $\beta_j$, $\alpha_j$ and structuring functions $b_{j}$ are learnt by backpropagation \cite{LEMD2022}. The learnable structuring functions $\mathbf{b}_{j}$ play the same role as the kernel in the convolutions. Note that one can have $R=1$, the pooling does not involve downsampling. 
We note that in a DCNN network the output of each layer $T^{k}$ is composed of the affine function $x\mapsto \mathbf{W}^{k}x + \mathbf{b}^{k}$, where $\mathbf{W}^{k}$ is the weight matrix (convolution weights in a CNN layer) and $\mathbf{b}^{k}$ the bias, and the activation function $\sigma$, i.e., $T^{k} =$ $\sigma\left( \mathbf{W}^{k} T^{k-1}  + \mathbf{b}^{k}\right)$, where $\sigma$ is acting elementwise. Using our general activation~\eqref{GeneralActivation}, we obtain that
\begin{equation*}
T^{k} =\bigwedge_{j\in J}\left[ \bigvee_{i\in I}\left\{\beta_i^{jk} \mathbf{W}^{k} T^{k-1}  + \beta_i^{jk} \mathbf{b}^{k} + \alpha_i^{jk}\right\}\right] ,
\end{equation*}
and therefore the bias has two terms which are learnt. We propose therefore to consider in our experiments that $\mathbf{b}^{k}$ is set to zero since its role will be replaced by learning the $\alpha_i^{jk}$.

\section{Experimental Section}

Firstly, to illustrate the kind of activation functions that our proposition can learn, we use the MNIST dataset as a ten class supervised classification problem and an architecture composed of two convolutional layers and a dense layer for reducing to the number of classes. The activation functions that we optimise by stochastic gradient descent have as general form $\min ( \max (\beta_0 x+\alpha_0,\beta_1 x+\alpha_1,\alpha_2),\alpha_3)$, which corresponds to \eqref{MorphoAct1} and \eqref{MorphoAct2} where $R=1$, i.e., without pooling. We have initialised all the activation functions to be equal to $\max(\min(\relu(x),6),-6)$ as it is illustrated in Fig.\ref{fig:VisualizationActivation}(left). The accuracy of this network without any training is $14\%$. Surprisingly when one optimises \footnote{We use ADAM optimizer with a categorical entropy as loss function, a batch size of 256 images and a learning rate of 0.001.} \emph{only} the parameters of activation functions the network accuracy increases to the acceptable performance of $92.38\%$ and a large variability of activation functions are found Fig.\ref{fig:VisualizationActivation}(center).
This is a way to assess the expressive power\footnote{The expressive power describes neural networks ability to approximate functions.} of the parameter of the activation as it has been proposed in \cite{frankle2020training}.
Additionally, an adequate separation among classes is noted by visualising the projection to two-dimensional space of the last layer via the t-SNE \cite{van2008visualizing} algorithm. Of course, a much better accuracy ($98,58\%$) and inter-class separation is obtained by optimising all the parameters of the network Fig.\ref{fig:VisualizationActivation}(right). 
\begin{figure}
    \centering
    \leavevmode
    \includegraphics[width=.32\columnwidth,bb=0 0 362 265]{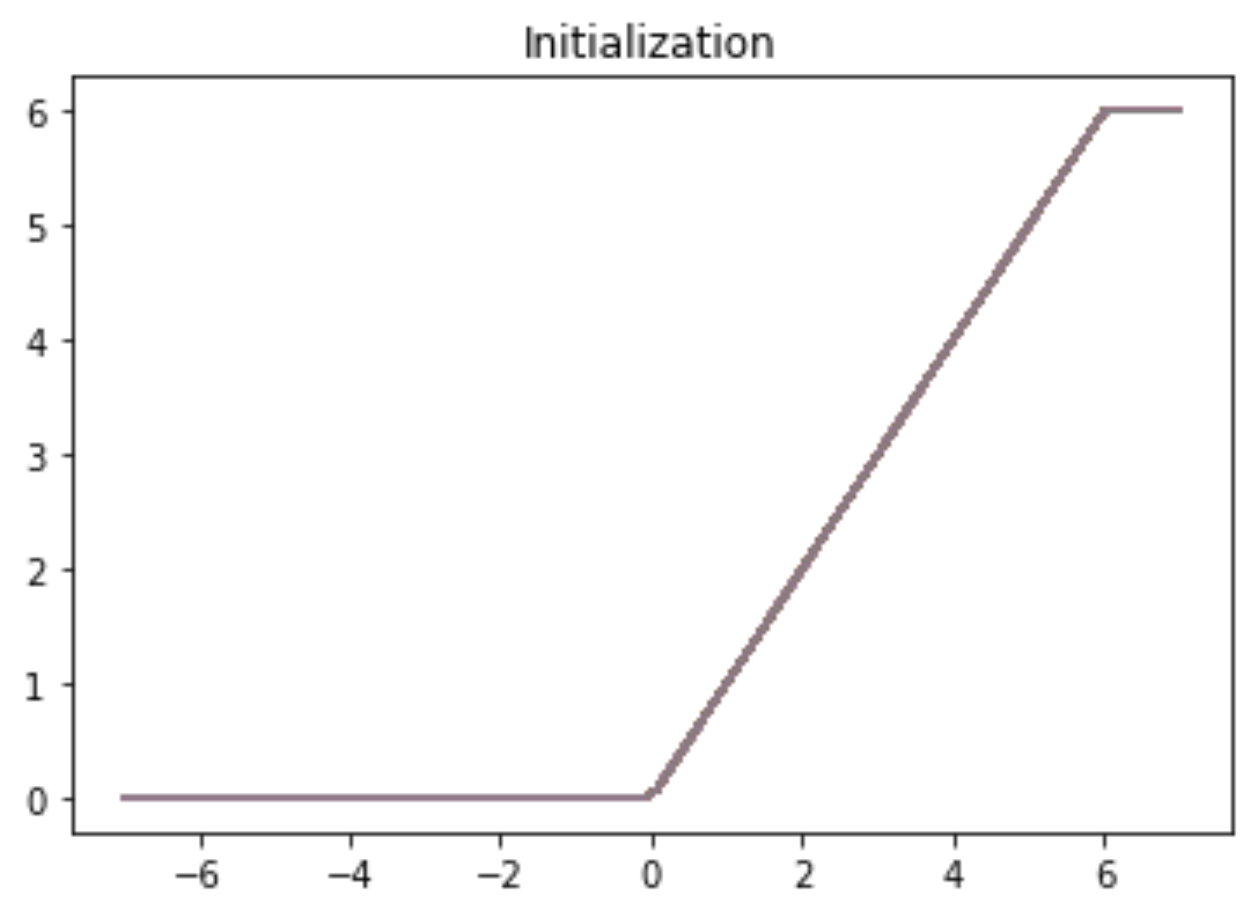}
    \includegraphics[width=.32\columnwidth,bb=0 0 371 265]{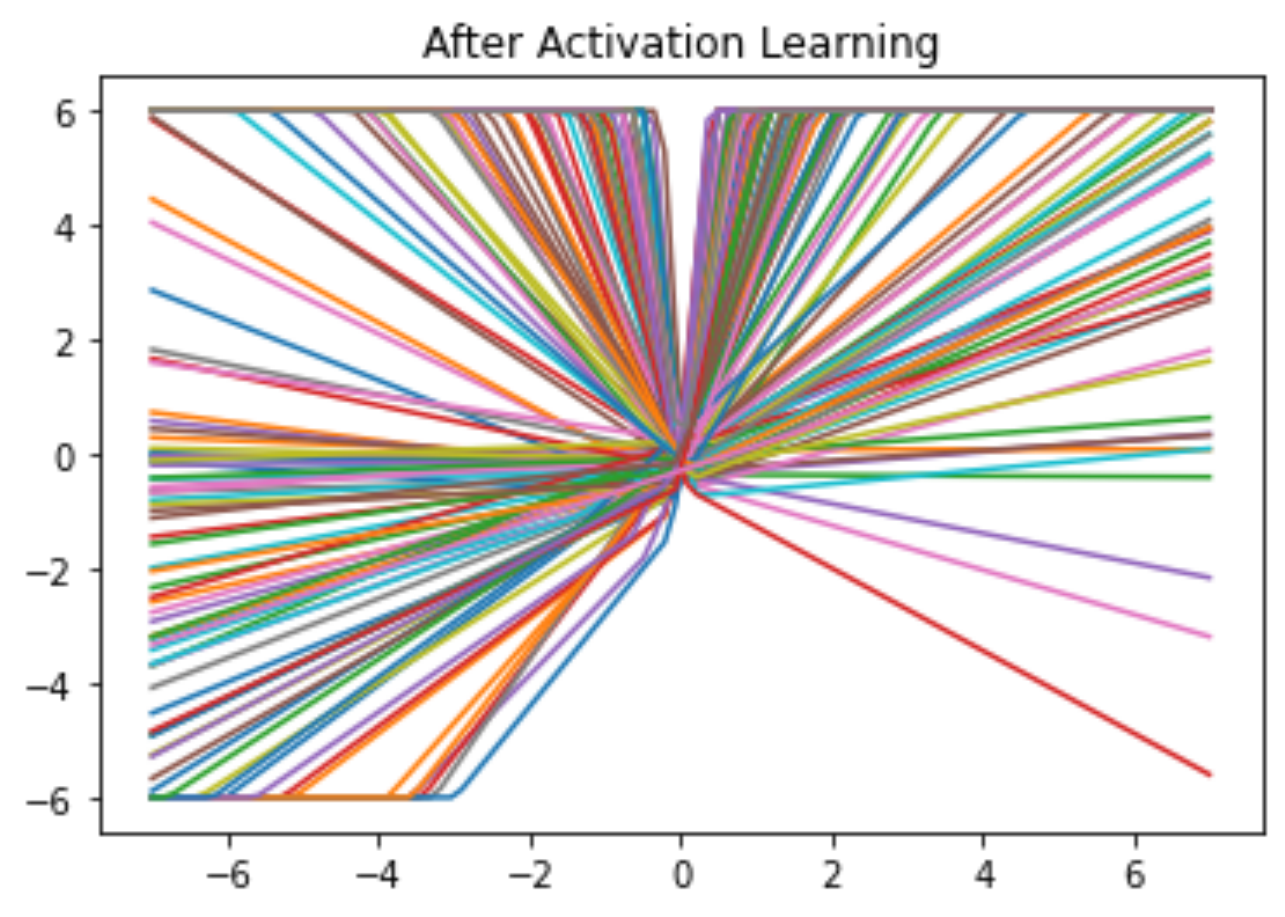}
    \includegraphics[width=.32\columnwidth,bb=0 0 371 265]{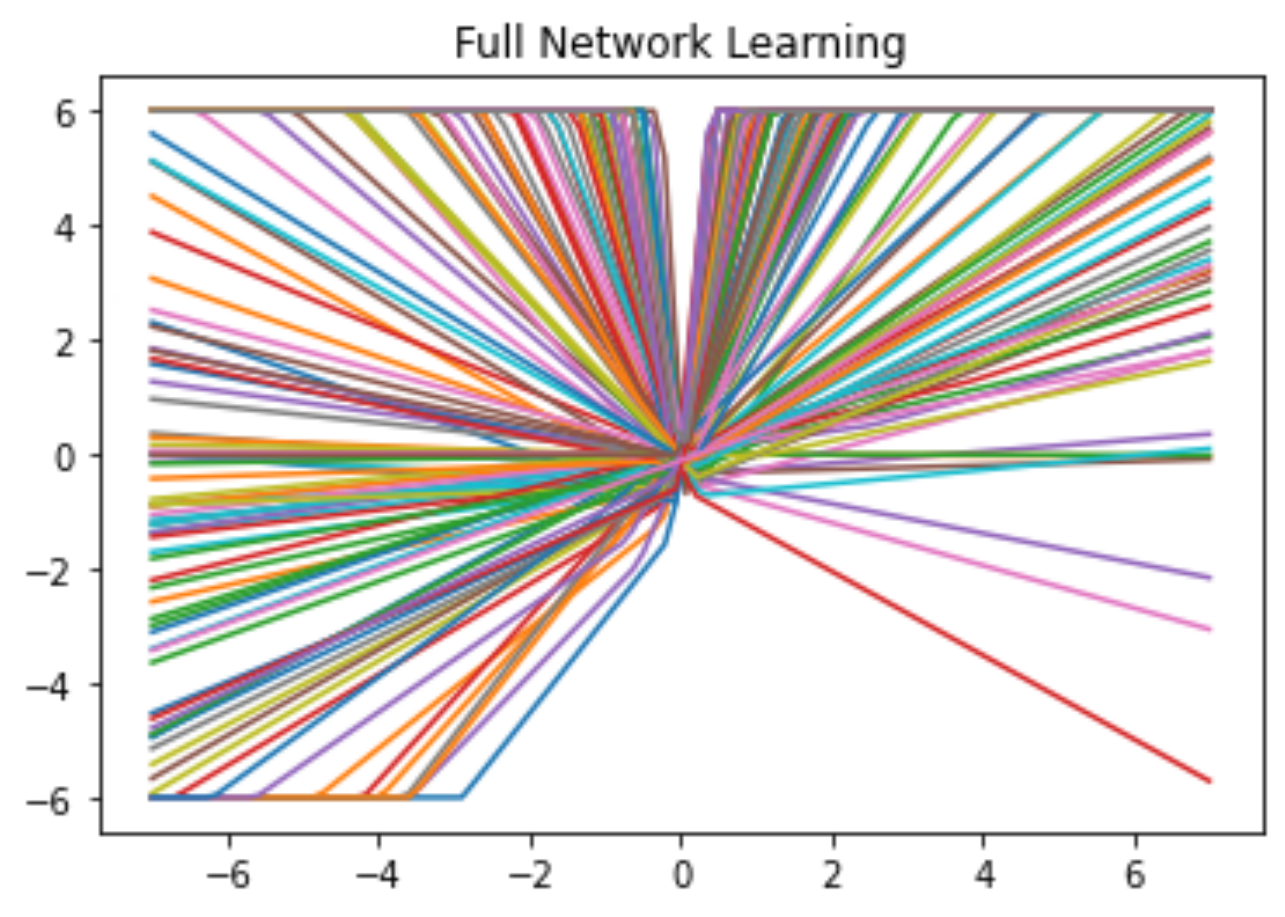}
    \includegraphics[width=.32\columnwidth,bb= 0 0 488 467]{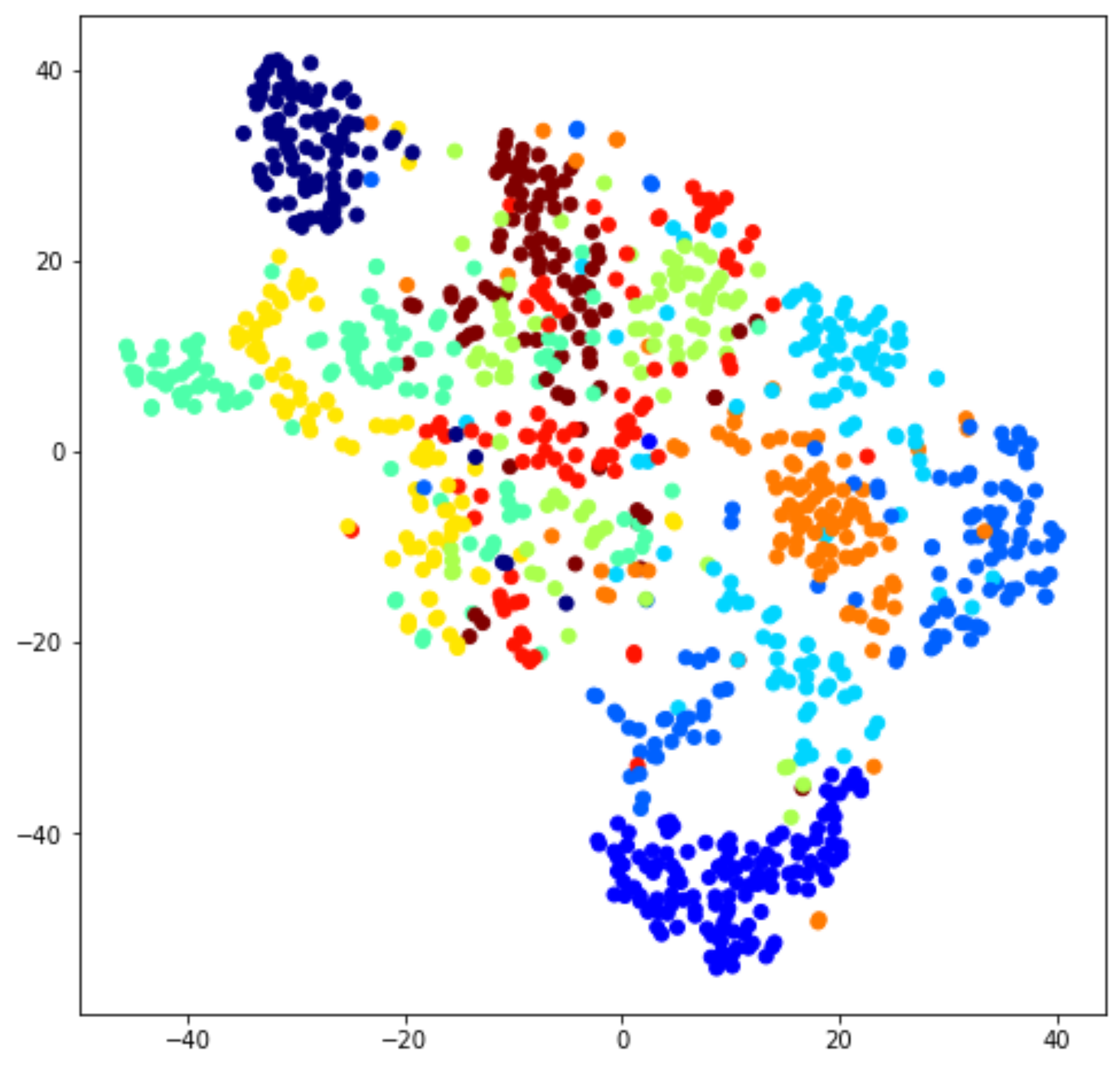}
    \includegraphics[width=.32\columnwidth,bb=0 0 488 467]{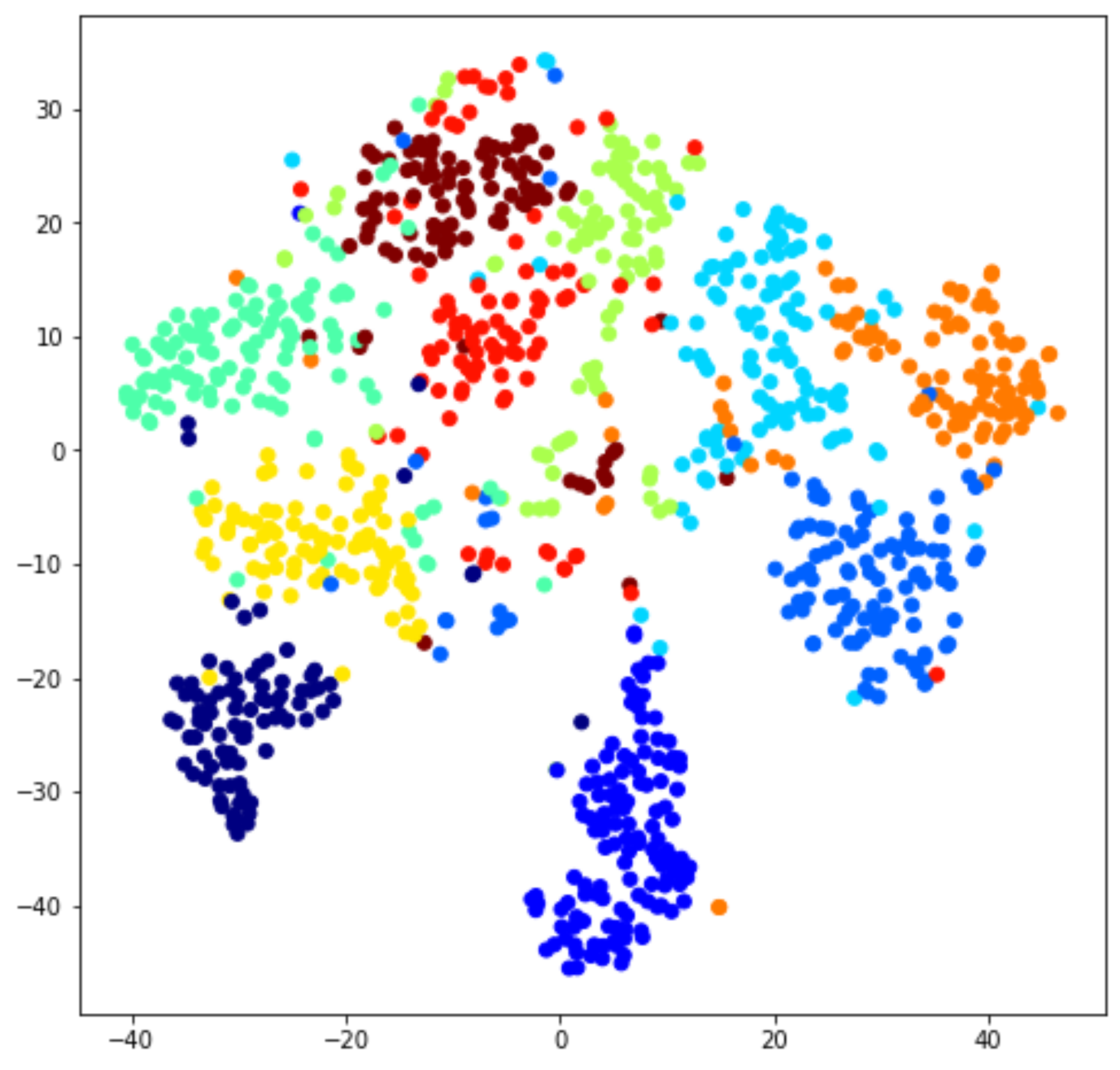}
    \includegraphics[width=.32\columnwidth,bb=0 0 488 467]{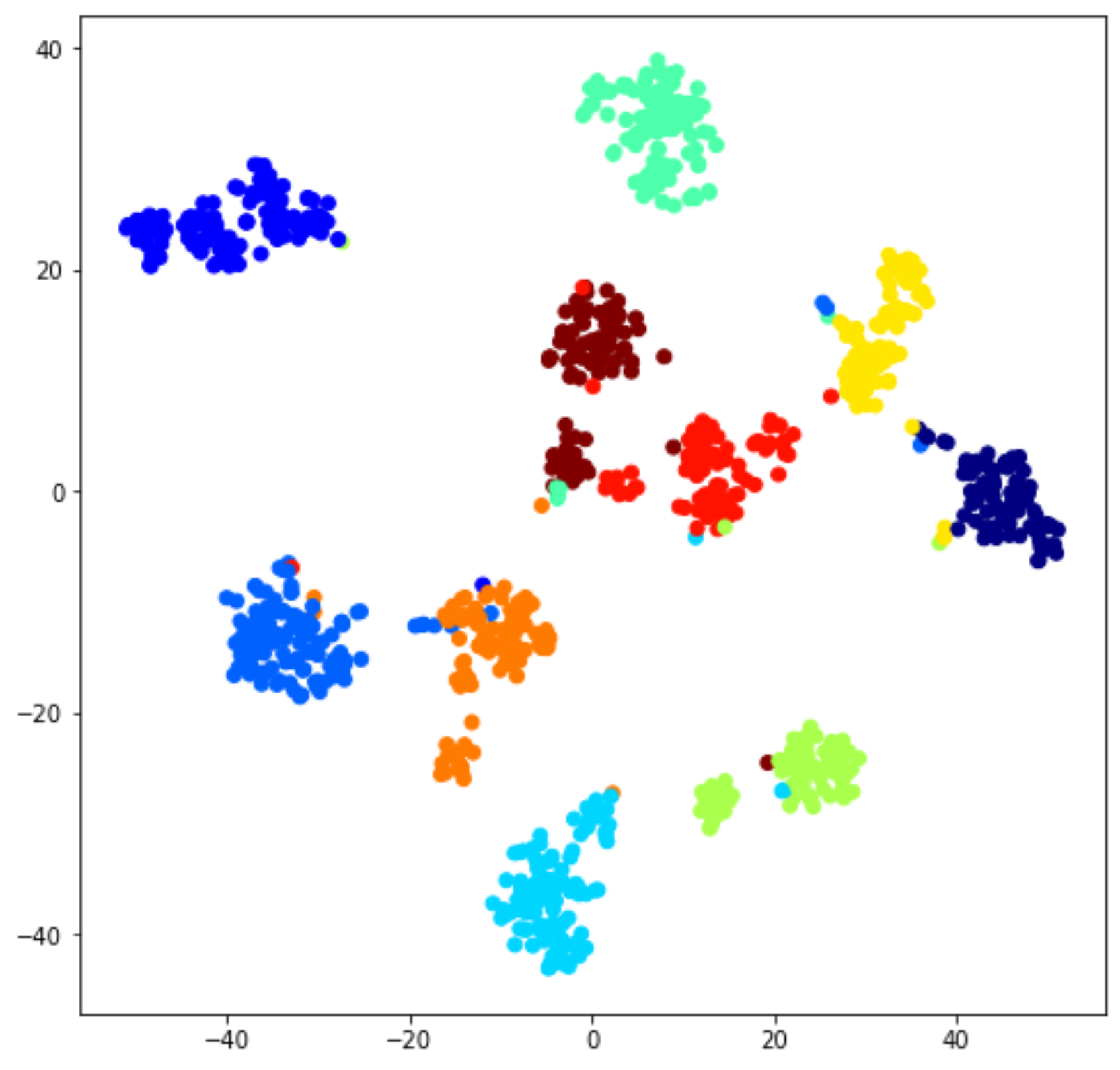}
    \caption{First Row: Left: Random Initialisation with (14\%) of accuracy on the test set, We use a simplified version of proposed activation $\min ( \max (\beta_0 x+\alpha_0,\beta_1 x+\alpha_1,\alpha_2),\alpha_3)$, with initialisation $\max(\min(\relu(x),6),-6)$ Centre: Training only activation functions (92.38\%), Right: Training Full Network (98,58\%). Second Row: t-SNE visualisation of last layer is the 10-classes MNIST prediction for a CNN. }
    \label{fig:VisualizationActivation}
\end{figure}


Secondly,
we compare the performance of \eqref{eq:SelfDualRelu}, \eqref{eq:SelfDualPosNegPar}, \eqref{MorphoAct1} and \eqref{MorphoAct2} following the common practice and train all the models using a training set and report the standard top-one error rate on a testing set. We use as architecture a classical two-layer CNN (without bias for \eqref{MorphoAct1} and \eqref{MorphoAct2}) with 128 filters of size ($3\times3$) per layer, and a final dense layer with dropout. After each convolution the different propositions are used to both produce a nonlinear mapping and reduce spatial dimension via pooling stride of two. As a manner of comparison, we include the case of a simple $\relu$ activation followed by a $\maxpool$ with stride two. The difference in top-one error rate on a testing set is reported in Table \ref{tab:ALL100Morpho} for CIFAR10, CIFAR100 and Fashion-MNIST databases. These quantitative results shown in propositions \eqref{eq:SelfDualRelu} and \eqref{eq:SelfDualPosNegPar} do not seem to improve the performance in the explored cases. Additionally, \eqref{MorphoAct2} performs better than  \eqref{MorphoAct1}, and  it improves the accuracy in comparison with our baseline in all the considered databases.

\begin{table}[]
    \centering
    \begin{tabular}{|c|c|c|c|c|c|c|c|c|c|}
    \hline
    \multicolumn{1}{|c|}{} & \multicolumn{3}{|c|}{Fashion MNIST} 
    & \multicolumn{3}{|c|}{CIFAR10} &
    \multicolumn{3}{|c|}{CIFAR100}\\
    \hline 
    \maxpool(\relu)  & \multicolumn{3}{|c|}{93.11} &
    \multicolumn{3}{|c|}{78.04} & \multicolumn{3}{|c|}{47.57}
    \\
    \hline 
    Self-dual Relu in  \eqref{eq:SelfDualRelu} & \multicolumn{3}{|c|}{-2.11} &
    \multicolumn{3}{|c|}{-20.12} & \multicolumn{3}{|c|}{-31.14}
    \\
    \hline 
    \eqref{eq:SelfDualPosNegPar}  & \multicolumn{3}{|c|}{-0.95} &
    \multicolumn{3}{|c|}{-1.75} & \multicolumn{3}{|c|}{-4.39}
    \\
    \hline
    MorphoActivation in \eqref{MorphoAct1} &  N=2      &  N=3   & N=4 &  N=2      &  N=3   & N=4 &  N=2      &  N=3   & N=4  \\ \hline
     M=2    &  -0.06 & -0.05 & -0.1 &   -0.42 &  0.02 & -0.02 & 0.44 & 0.7  & 0.4\\ \hline
     M=3    & -0.14 & -0.14 & -0.06 & -0.57 & -0.4  & -0.35  & 0.56 & 0.49 & 0.61\\  \hline
     M=4    & -0.02 & -0.08 & -0.01 & 0.05 & -0.62 & -0.5 & 0.41 & 0.35 & 0.73\\  \hline
    MorphoActivation in \eqref{MorphoAct2} &  N=2      &  N=3   & N=4 &  N=2      &  N=3   & N=4 &  N=2      &  N=3   & N=4 \\ \hline
     M=2    & 0.04 & -0.16 & -0.12  &  1.84 & 2.02 & 1.49 & 3.31 & 3.5 & 3.45\\ \hline
     M=3    & 0.08 & -0.09 &  \textbf{0.12} & 2.39 & 1.96 & 1.82 & 3.48 & 3.55 & \textbf{3.86}\\  \hline
     M=4    & -0.02 & 0.09 & -0.03 & \textbf{2.49} & 2.25 & 2.13 & 3.47 & 3.73 & 3.58\\  \hline
    \end{tabular}
    \caption{Relative difference with respect to our baseline ($\relu$ followed by a $\maxpool$). Architecture used is a CNN with two layers. ADAM optimiser with an early stopping with patience of ten iterations. Only Random Horizontal Flip has been used as image augmentation technique for CIFARs. The results are the average over three repetitions of the experiments.}
    \label{tab:ALL100Morpho}
\end{table}

\section{Conclusions and Perspectives}
\label{sec:Conclusions}

To the best of our knowledge, this is the first work where nonlinear activation functions in deep learning are formulated and learnt as max-plus affine functions or tropical polynomials. We have also introduced an algebraic framework inspired from mathematical morphology which provides a general representation to integrate the nonlinear activation and pooling functions.  

Besides more extended experiments on the performance on advanced DCNN networks, our next step will be to study the expressivity power of the networks based on our morphological activation functions. The universal approximation theorems for ReLU networks would just be a particular case. We conjecture that the number of parameters we are adding on the morphological activation can provide a benefit to get more efficient approximations of any function with the same width and depth.

\subsubsection*{Acknowledgements}
This work has been supported by Fondation Mathématique  Jacques Hadamard (FMJH) under the PGMO-IRSDI 2019 program.
This work was granted access to the Jean Zay supercomputer under the allocation 2021-AD011012212R1.
\bibliographystyle{splncs}
\bibliography{bibliography}

	


\end{document}